\documentclass{article}

 \usepackage[main, final]{neurips_2026}

\usepackage[utf8]{inputenc} 
\usepackage[T1]{fontenc}    
\usepackage{hyperref}       
\usepackage{url}            
\usepackage{booktabs}       
\usepackage{amsfonts}       
\usepackage{nicefrac}       
\usepackage{microtype}      
\usepackage{xcolor}         

\usepackage{amsmath}
\usepackage{amssymb}
\usepackage{mathtools}
\usepackage{amsthm}
\usepackage{multirow}
\usepackage{comment}
\setcitestyle{round}
\theoremstyle{plain}
\newtheorem{theorem}{Theorem}[section]
\newtheorem{proposition}[theorem]{Proposition}

\theoremstyle{definition}

\newtheorem{example}{Example}
\newtheorem{problem}{Problem}
\newtheorem*{problem*}{Problem}
\theoremstyle{remark}

\title{Prefix-Tuning for Arbitrary Output Sequences on Pretrained Transformers}

\author{%
  Peter Cho-Ho Lam \quad
  Ziyi Wang \quad
  Zirui Zhou\\
  Huawei Technologies Canada Co., Ltd.\\
  \texttt{cho.ho.lam@huawei.com} \quad
  \texttt{ziyi.wang2@h-partners.com} \quad
  \texttt{zirui.zhou@huawei.com}\\
  \url{https://github.com/wzii/prefix_tuning}
}

\begin{document}

\maketitle

\begin{abstract}
Prefix-tuning is a parameter-efficient fine-tuning method for large language models (LLMs) that steers model outputs by prepending task-specific prefix representations to the input sequence. Despite its practical success, the theoretical understanding of the expressiveness of prefixes remains incomplete. In particular, the following question remains open: {\it Given a pretrained autoregressive Transformer, for any input-output sequence pair $(\mathtt{x},\mathtt{y})$, does there always exist a prefix $\mathcal{S}$ such that prepending $\mathcal{S}$ to $\mathtt{x}$ induces the model to generate the target output $\mathtt{y}$?} In this work, we prove that under the assumption that $\mathtt{y}$ is in the {\it range} of the pretrained Transformer, such a prefix $\mathcal{S}$ exists and can be explicitly computed by solving multiple linear systems. We further show that this assumption is necessary, in the sense that such a prefix may not exist when the assumption fails, thereby providing a complete answer to the above question. Our analysis is constructive and the obtained theoretical results hold under fairly general assumptions on model architecture. Extensive experiments on GPT-2 and Qwen3-4B demonstrate that our algorithm produces prefixes with strong control over model outputs. 

\end{abstract}

\section{Introduction}
Large language models (LLMs) have become a major driving force behind modern technological advancement. Their effectiveness has been demonstrated across a wide range of applications, including software engineering, scientific research, healthcare, finance, education, and legal domains \citep{Bommasani2021OnTO, Chen2024ASO, Yang2025SurveyOS}. However, most LLMs are trained as general-purpose models, which can limit their performance on domain-specific tasks that require specialized knowledge or terminology. Empirical studies consistently show that models fine-tuned for a target domain outperform general-purpose models on domain-specific benchmarks and task-relevant metrics \citep{beltagy-etal-2019-scibert, Lee2019BioBERTAP, chalkidis-etal-2020-legal, gururangan-etal-2020-dont}.

Although fine-tuning is substantially cheaper than training a model from scratch, full fine-tuning remains computationally expensive due to the need to update and store all model parameters, making it impractical for many users and organizations \citep{strubell-etal-2019-energy}. These limitations have motivated growing interest in parameter-efficient fine-tuning methods, which aim to reduce training costs while retaining performance \citep{He2021TowardsAU, Ding2023ParameterefficientFO, Xu2023ParameterEfficientFM, Han2024ParameterEfficientFF}. Parameter-efficient fine-tuning techniques update a limited subset of parameters or introduce lightweight modules, achieving performance comparable to full fine-tuning at a fraction of the computational cost. \citep{Houlsby2019ParameterEfficientTL, BenZaken2021BitFitSP, Davison2021CompacterEL, Hu2021LoRALA, lester-etal-2021-power, LL, Lian2022ScalingS}.

Prefix-tuning, a prominent parameter-efficient fine-tuning technique, adapts new tasks by learning a small set of continuous task-specific vectors, or prefixes, which are prepended to the model’s hidden states while keeping the original parameters frozen \citep{LL}. By optimizing only these prefixes, prefix-tuning avoids the high computational and memory costs of full fine-tuning. This efficiency makes prefix-tuning a practical and scalable approach for adapting large models to diverse tasks and deployment settings \citep{liu-etal-2022-p}.

Despite its popularity, the theoretical foundations of prefix-tuning remain poorly understood. A desirable property of prefix-tuning is \textbf{expressiveness}: Can one induce arbitrary desired behaviors in a Transformer solely by appending prefixes?  Formally, this can be framed as the following problem:

\begin{problem}
\label{prob}
Given a pretrained autoregressive Transformer, for any input-output sequence pair $(\mathtt{x},\mathtt{y})$, does there always exist a prefix $\mathcal{S}$ such that prepending $\mathcal{S}$ to $\mathtt{x}$ induces the model to generate the target output $\mathtt{y}$?
\end{problem}

Several existing works have studied the expressiveness of prefix-tuning \citep{Wang2023UniversalityAL, PTB1, HuWGL0025, Liang2025TowardsIP}; however, their analyses suffer from three notable limitations. (i) \textbf{Model mismatch: } Their results are established for hand-crafted Transformers with preassigned weights and simplified architectures, rather than pretrained models specified by the user. (ii) \textbf{Generation mismatch: } A number of works study the approximation of arbitrary sequence-to-sequence functions \citep{PTB1, HuWGL0025}, whereas in practice language models generate outputs autoregressively through repeated queries to the Transformer. This autoregressive setting imposes significantly stronger structural constraints than a one-shot sequence mapping. \footnote{For a more thorough discussion on why sequence-to-sequence function is incompatible with autoregressive generation, see our Appendix \ref{seq-to-seq}.} (iii) \textbf{Scaling assumptions:} The resulting guarantees typically rely on impractical assumptions regarding model width, depth, or prefix length \citep{PTB1, Liang2025TowardsIP}. Due to these limitations, expressiveness results have never been empirically demonstrated on modern LLMs.

In this work, we address this gap by establishing an expressiveness result for pretrained autoregressive Transformers with general architectural components, where the model is fixed a priori and cannot be modified. Specifically, we prove the existence of $\mathcal{S}$ in Problem \ref{prob} under the assumption that $\mathtt{y}$ lies in the \textit{range} of the Transformer; that is, $\mathtt{y}$ is a sequence reachable through autoregressive generation. With minor adaptations, our proof applies to nearly all common modern LLM architectures, including GPT-2 \citep{Radford2019LanguageMA}, which features multi-head attentions, residual connections, and positional encodings. The statement of our main theorem (Theorem \ref{main_theorem}) is theoretically tight: removing any assumption leads to counterexamples, implying that our result is best possible and that Problem \ref{prob} is fully resolved.


Our proof also provides a practical algorithm to compute such a prefix. The computation is training-free and amounts to solving a fixed and small number of low-dimensional systems of linear equations, which can be done efficiently with minimal hardware requirements. The construction of our linear-system formulation is based on attention matching, a framework that matches attention outputs across different autoregressive decoding trajectories. We note that a related attention matching framework was independently developed by \cite{zweiger2026fast} for KV cache compaction. We discuss the common underlying principles and the differences between the two frameworks in Appendix~\ref{app:zweiger}.


To validate our theory and algorithm, we conducted extensive experiments using GPT-2 \citep{Radford2019LanguageMA}. Our results demonstrate a high success rate in achieving complete control of output sequences through prefixes. To the best of our knowledge, this is the first instance in which a theoretical expressiveness result has been empirically demonstrated and validated in a modern LLM. We also achieved strong results on Qwen3-4B \citep{qwen3}, which features grouped-query attention (GQA) \citep{GQA}.


Below is a summary of our contributions:

\begin{itemize}
\item We resolve a long-standing open question on the expressiveness of prefix-tuning in practical LLMs and fine-tuning scenarios by providing and proving the necessary and sufficient conditions on the existence of a prefix that can generate essentially any arbitrary output.
\item We provide an efficient, training-free algorithm for constructing such prefixes and validate it through extensive experiments on GPT-2 and Qwen3-4B, showing that the resulting prefixes achieve a high success rate in fully controlling output sequences.
\end{itemize}

\section{Background}

\subsection{Transformer Notations}
Let $\mathcal{V}$ be a finite vocabulary, and let $|\mathcal{V}|$ be its cardinality. A token $\mathtt{x}$ is an element of $\mathcal{V}$. A sequence of tokens of length $p$ is denoted by $(\mathtt{x}_1, \dots, \mathtt{x}_p)$.

Fix an embedding dimension $d_e$. The embedding map is a function $\textbf{Emb}:\mathcal{V}\rightarrow\mathbb{R}^{d_e}$ applied element-wise to a token sequence. In practice, $\textbf{Emb}$ may also incorporate positional information. Given a sequence $(\mathtt{x}_1, \dots, \mathtt{x}_p)$, we define the embedded input as $\mathbf{x}^{(0)}_i = \textbf{Emb}(\mathtt{x}_i)$ for $i=1, \dots, p$.

An $L$-layer Transformer $\mathcal{T}$ is a parameterized mapping $\mathcal{T}:\bigcup_{p\ge1}(\mathbb{R}^{d_e})^p\longrightarrow\bigcup_{p\ge1}(\mathbb{R}^{|\mathcal{V}|})^p$
in the form $\mathcal{T} = \mathcal{F}^{(L)}\circ\dots\circ\mathcal{F}^{(1)}$
where, for each layer $l = 1, \dots, L$, $\mathcal{F}^{(l)} = \textbf{MLP}^{(l)}\circ\mathbf{MHA}^{(l)}$. Here, $\textbf{MHA}^{(l)}$ is a multi-head attention module and $\textbf{MLP}^{(l)}$ is a position-wise multi-layer perceptrons, each with layer-specific parameters.

Given an output of $\mathcal{T}$ in $(\mathbb{R}^{|\mathcal{V}|})^p$, a deterministic next-token prediction may be obtained by applying the $\textbf{argmax}$ operator to the representation \textbf{at the final position}. By a slight abuse of notation, we will also use $\mathcal{T}$ to denote the composition of the Transformer with this $\textbf{argmax}$ operation, in which case its output space is $\mathcal{V}$. The generated token is then appended to the input sequence and fed back into $\mathcal{T}$ to produce subsequent tokens iteratively. This procedure, known as \textbf{autoregressive decoding}, is the standard mode of operation for LLMs and the default decoding method considered in this paper.

Given the embedded input $(\mathbf{x}^{(0)}_1, \dots, \mathbf{x}^{(0)}_p)$, the layer-wise representations are defined recursively by
\[
(\mathbf{x}^{(l)}_1, \dots, \mathbf{x}^{(l)}_p) = \mathcal{F}^{(l)}(\mathbf{x}^{(l-1)}_1, \dots, \mathbf{x}^{(l-1)}_p)
\]
for $l=1, \dots, L$. For notational simplicity, we omit the layer index and use $\mathbf{x}_i$ for a generic token representation. Note that for all $l\ge1$, the representation $\mathbf{x}^{(l)}_i$ depends on the previous tokens through the attention mechanism. This dependence is also suppressed in the notation.

We now describe a single multi-head attention module $\mathcal{A}$ in detail. Let $H$ denote the number of attention heads. For $h\in\{1, \dots, H\}$, let $\mathbf{W}^h_Q, \mathbf{W}^h_K\in\mathbb{R}^{d_{\text{hidden}}\times d_{\text{in}}}, \mathbf{W}^h_V\in\mathbb{R}^{d_{\text{out}}\times d_{\text{in}}}$ be the query, key, and value matrices, respectively. Given an input sequence of vectors $(\mathbf{x}_1, \dots, \mathbf{x}_p)$ with $\mathbf{x}_i\in\mathbb{R}^{d_\text{in}}$, the attention matrix for head $h$ is $\mathbf{A}^h\in\mathbb{R}^p\times\mathbb{R}^p$, whose entries are given by
\[
\mathbf{A}_{i, j}^h = \frac{\exp(\frac{T}{\sqrt{d_\text{hidden}}}(\mathbf{W}^h_Q\mathbf{x}_i)^\top(\mathbf{W}^h_K\mathbf{x}_j))}{\sum_{r=1}^p\exp(\frac{T}{\sqrt{d_\text{hidden}}}(\mathbf{W}^h_Q\mathbf{x}_i)^\top(\mathbf{W}^h_K\mathbf{x}_r))},
\]
where $T>0$ is an inverse temperature parameter. Then, the \textbf{MHA} output at position $i$ is given by:\footnote{Alternatively, concatenation can be used instead of summation. It does not affect our theoretical analysis.}
\[
\mathbf{a}_i = \sum_{h=1}^H\sum_{j=1}^p\mathbf{A}^h_{i, j}\mathbf{W}^h_V\mathbf{x}_j.
\]
We also write this as $\mathcal{A}((\mathbf{x}_1, \dots, \mathbf{x}_p)) = (\mathbf{a}_1, \dots, \mathbf{a}_p)$. For simplicity, we omit the potential bias term in our expressions; see Appendix \ref{appendix_bias} for treatment with biases. 


\subsection{Prefix-Tuning}
\label{sec:prefix}
In the literature and in practical implementations, there exist different levels of generality in how prefixes are defined. Here, we adopt a formulation that encompasses most existing variants while remaining mathematically concise.

Formally, let $\mathcal{T}$ be an $L$-layer Transformer. A prefix $\mathcal{S}$ with respect to $\mathcal{T}$ is a set of layer-specific sequences of vectors
\[
\mathcal{S} = \{(\mathbf{s}^{(l)}_1, \dots, \mathbf{s}^{(l)}_m)\in(\mathbb{R}^{d_\text{in}})^m~|~l=1, \dots, L\}.
\]
For the $l$-th layer, the input to the multi-head attention module $\mathcal{A}^{l}$ is then
\[
(\mathbf{s}^{(l)}_1, \dots, \mathbf{s}^{(l)}_m, \mathbf{x}^{(l-1)}_1, \dots, \mathbf{x}^{(l-1)}_p).
\]

Many existing works introduce additional flexibilities in the definition of the prefix. For example, some allow the vectors $\mathbf{s}^{(l)}_i$ to be specific to each attention head \citep{adapter_library}, while others allow separate prefixes to be applied when projecting to the key and value matrices,\footnote{Technically, multiplying with $\mathbf{W}^h_Q$ is not necessary for inference.} $\mathbf{W}^h_K$ and $\mathbf{W}^h_V$ \citep{PTB1}. For the purposes of this work, we initially adopt the simpler, layer-wise prefix formulation and will generalize to these more flexible definitions when required in our theoretical analysis. We refer to these constructions as \textbf{generalized} prefixes in our main results. In practice, the choice of prefix variant often has little impact on memory usage or inference speed as they can be implemented in the form of KV caches.


\subsection{Main Conjecture}

Our goal is to establish the expressiveness of prefix-tuning in a practical large language model setting, where outputs are generated autoregressively and both the pretrained Transformer and the target dataset are fixed in advance. Before doing so, however, it is important to clarify a fundamental limitation: in general, for a given Transformer $\mathcal{T}$, not all sequences over the vocabulary can be generated autoregressively. For instance, it is straightforward to construct a Transformer for which a particular token never attains the maximum logit and thus is never produced. Even if a Transformer is surjective at the level of individual tokens—meaning that every vocabulary can be generated at some position—this does not imply that every sequence of tokens can be realized.

Consequently, to meaningfully discuss the expressiveness of prefix-tuning, we restrict attention to output sequences that are realizable by the underlying Transformer. Specifically, for a target output sequence $\mathtt{y}=(\mathtt{y}_1, \dots, \mathtt{y}_k)$, we always assume that there exists some source input sequence $\mathtt{x}'=(\mathtt{x}'_1, \dots, \mathtt{x}'_p)$ such that $\mathcal{T}$ produces $(\mathtt{y}_1, \dots, \mathtt{y}_k)$ when queried autoregressively on this input. In other words, $\mathtt{y}$ is in the \textit{range} of $\mathcal{T}$. This assumption is consistent with the findings of \citep{PTB2}, which show that a prefix-tuned Transformer cannot exhibit behaviors that are not already present in the original model. In what follows, for a prefix $\mathcal{S}$, we write $\mathcal{T}_\mathcal{S}$ to denote the Transformer obtained by applying the prefix $\mathcal{S}$ to $\mathcal{T}$. We now state our main problem, which is a precise formulation of Problem \ref{prob}:

\begin{problem}
\label{input_output_conjecture}
Let $\mathcal{T}$ be a Transformer and $\mathtt{y}=(\mathtt{y}_1, \dots, \mathtt{y}_k)$ be a sequence of tokens for which there exists a sequence of tokens $\mathtt{x}'=(\mathtt{x}'_1, \dots, \mathtt{x}'_p)$ such that
\begin{equation}
\label{ioc_assumption}
\begin{aligned}
\mathcal{T}((\mathtt{x}'_1, \dots, \mathtt{x}'_p)) &= \mathtt{y}_1, \\
\mathcal{T}((\mathtt{x}'_1, \dots, \mathtt{x}'_p, \mathtt{y}_1)) &= \mathtt{y}_2, \\
&\vdots\\
\mathcal{T}((\mathtt{x}'_1, \dots, \mathtt{x}'_p, \mathtt{y}_1, \mathtt{y}_2, \dots, \mathtt{y}_{k-1})) &= \mathtt{y}_k.
\end{aligned}
\end{equation}
Then, for any sequence of tokens $\mathtt{x}=(\mathtt{x}_1, \dots, \mathtt{x}_n)$, does there always exist a prefix $\mathcal{S}$ of $\mathcal{T}$ such that
\begin{equation}
\label{ioc_goal}
\begin{aligned}
\mathcal{T}_\mathcal{S}((\mathtt{x}_1, \dots, \mathtt{x}_n)) &= \mathtt{y}_1, \\
\mathcal{T}_\mathcal{S}((\mathtt{x}_1, \dots, \mathtt{x}_n, \mathtt{y}_1)) &= \mathtt{y}_2, \\
&\vdots\\
\mathcal{T}_\mathcal{S}((\mathtt{x}_1, \dots, \mathtt{x}_n, \mathtt{y}_1, \mathtt{y}_2, \dots, \mathtt{y}_{k-1})) &= \mathtt{y}_k?
\end{aligned}
\end{equation}
\end{problem}
We refer to $(\mathtt{y}_1, \dots, \mathtt{y}_k)$ as the \textbf{target output}, $(\mathtt{x}'_1, \dots, \mathtt{x}'_p)$ as the \textbf{source input}, and $(\mathtt{x}_1, \dots, \mathtt{x}_n)$ as the \textbf{user input}. The conclusion of this problem is extremely strong: it asserts that for an arbitrary user input, there exists a prefix that steers the model’s output towards the target sequence, even if it is not related to the user input. Our main result resolves this problem under fairly general conditions.

\begin{theorem}
\label{main_theorem}
There exists a (generalized) prefix that satisfies \eqref{ioc_goal} if there exists a source input $\mathtt{x}'$ for $\mathtt{y}$ and if condition \eqref{Vandermonde_assumption} holds at every attention head. If $\mathcal{T}$ has residual connections and positional encodings, then the conditions $p=n$ and $\mathtt{x}_n=\mathtt{x}'_p$ are also required.
\end{theorem}

Condition \eqref{Vandermonde_assumption} is primarily used to facilitate certain linear-algebraic operations in our proof and is never found to be violated in our experiments. Beyond this, Theorem \ref{main_theorem} is best possible, as the source input assumption cannot be removed. For the case of residual connections and positional encodings, the statement is also best possible as counterexamples exist when the assumptions $p=n$ and $\mathtt{x}_n=\mathtt{x}'_p$ are removed (see Appendix \ref{2-element}). Therefore, Theorem \ref{main_theorem} provides a complete answer to Problem \ref{input_output_conjecture}.

The requirement $\mathtt{x}_n=\mathtt{x}'_p$ may appear restrictive at first glance, but it is almost automatically satisfied in practice: in chatbot-style interactions, user inputs are typically terminated by a special end-of-conversation token inserted by the system, which serves as the common terminal token required here. Finally, the condition $p=n$ is introduced mainly to mitigate the effects of positional encodings. Since positional encodings typically induce only small perturbations to embedding vectors, this requirement is not essential in real applications. This claim is supported by our additional experiment in Appendix \ref{exper2}.

Note that Problem \ref{input_output_conjecture} corresponds to the fine-tuning setting with a single input-output pair. While it is natural to consider extending this formulation to larger datasets, such a generalization does not hold in general. In fact, even full fine-tuning may fail in this regime:

\begin{theorem}
\label{2_data}
There exists a fine-tuning dataset of two elements that cannot be realized even for full fine-tuning.
\end{theorem}

In other words, Theorem \ref{main_theorem} is also best possible with respect to the dataset size. Note that this theorem does not undermine the viability of prefix-tuning (or full fine-tuning). The construction in Theorem \ref{2_data} and its variants rely on fine-tuning datasets that exhibit contradictory behavior, either internally or with respect to the pretrained model, which is not representative of realistic fine-tuning scenarios. It would be of independent interest to formalize a notion of consistency and extend our results to larger, consistent datasets. For further discussion, see Appendix \ref{2-element}.

\section{Attention Matching}
\label{attn_matching}
In this section, we outline our strategy for constructing a prefix that satisfies \eqref{ioc_goal}. In principle, \eqref{ioc_goal} may be viewed as a system of high-dimensional nonlinear equations, for which general-purpose solvers exist. However, directly applying such methods faces several significant challenges: the dimensionality of the system is extremely large; the underlying Transformer may be non-differentiable, rendering gradient-based approaches inapplicable; and, most importantly, there is no a priori guarantee that a solution exists.

To address these difficulties, we develop an \textbf{attention matching} framework for prefix construction, which aims to match the attention outputs in \eqref{ioc_assumption} and \eqref{ioc_goal}. The central idea of this framework is to formulate and solve a collection of \textbf{matching equations}, consisting of one system of equations per attention head. By allowing the key and value prefixes to be decoupled, these matching equations admit a linear formulation. As a result, they can be solved efficiently using standard linear algebraic techniques.

We formally introduce the matching equations in Section \ref{match_eq} and present their linear formulations in Section \ref{linear}. The complete proof of Theorem \ref{main_theorem} based on the attention matching framework is deferred to Appendix \ref{proof_main}.

\subsection{Matching Equations}
\label{match_eq}
Consider the decoding stage corresponding to the source input $(\mathtt{x}'_1, \dots, \mathtt{x}'_p)$ and let $\mathcal{A}$ denote a multi-head attention module of the Transformer $\mathcal{T}$. During the autoregressive generation of the output tokens $\mathtt{y}_1, \dots, \mathtt{y}_k$, the following computations are performed by $\mathcal{A}$:
\begin{equation}
\label{attn_source}
\begin{aligned}
\mathcal{A}((\mathbf{x}'_1,\dots, \mathbf{x}'_p)) &= (\mathbf{a}_1, \dots, \mathbf{a}_p), \\
\mathcal{A}((\mathbf{x}'_1,\dots, \mathbf{x}'_p, \mathbf{y}_1)) &= (\mathbf{a}_1, \dots, \mathbf{a}_p, \mathbf{a}_{p+1}), \\
&\vdots\\
\mathcal{A}((\mathbf{x}'_1,\dots, \mathbf{x}'_p, \mathbf{y}_1, \dots, \mathbf{y}_{k-1})) &= (\mathbf{a}_1, \dots, \mathbf{a}_p, \mathbf{a}_{p+1}, \dots, \mathbf{a}_{p+k-1}). \\
\end{aligned}
\end{equation}
These computations occur sequentially at each step of the generation process of $\mathtt{y}_1, \dots, \mathtt{y}_k$.

Now, consider an arbitrary user input token sequence $(\mathtt{x}_1, \dots, \mathtt{x}_n)$, together with its corresponding embedding vectors $(\mathbf{x}_1, \dots, \mathbf{x}_n)$ immediately before entering $\mathcal{A}$. Our objective is to construct a prefix consisting of $m$ vectors $(\mathbf{s}_1, \dots, \mathbf{s}_m)$, where $m$ is a hyperparameter, such that
\begin{equation}
\label{attn_goal}
\begin{aligned}
\mathcal{A}((\mathbf{s}_1, \dots, \mathbf{s}_m, \mathbf{x}_1,\dots, \mathbf{x}_n)) =~& (\dots, \mathbf{a}_p), \\
\mathcal{A}((\mathbf{s}_1, \dots, \mathbf{s}_m, \mathbf{x}_1,\dots, \mathbf{x}_n, \mathbf{y}_1)) =~& (\dots, \mathbf{a}_p, \mathbf{a}_{p+1}), \\
&\vdots\\
\mathcal{A}((\mathbf{s}_1, \dots, \mathbf{s}_m, \mathbf{x}_1,\dots, \mathbf{x}_n, \mathbf{y}_1, \dots, \mathbf{y}_{k-1})) =~& (\dots, \mathbf{a}_p, \mathbf{a}_{p+1}, \dots, \mathbf{a}_{p+k-1})\\
\end{aligned}
\end{equation}
holds. Here we omit the first $m+n-1$ outputs, as they play no role in determining subsequent token predictions.

Equivalently, our goal is to ensure that, when generating the first $k$ output tokens using the user input augmented with the prefix, the attention outputs produced by $\mathcal{A}$ match those obtained when decoding from the source input. In other words, the outputs of $\mathcal{A}$ in the two runs—one conditioned on the source input and the other on the user input with prefix—are identical for all components that influence next-token prediction. If this condition is satisfied, then both runs will necessarily produce the same sequence of predictions for the first
$k$ tokens, namely $(\mathtt{y}_1, \dots, \mathtt{y}_k)$.\footnote{There are some nuisances to this discussion when residual connections are present. They will be discussed in Appendix \ref{proof_main}.}


The attention output in \eqref{attn_source} can be written explicitly as
\[
\begin{aligned}
\mathbf{a}_{p+t} =~& \sum_{h=1}^H\bigg(\sum_{i=1}^p\mathbf{A}^h_{t, i}\mathbf{W}^h_V\mathbf{x}'_i+\sum_{i=1}^{t}\mathbf{A}^h_{t, p+i}\mathbf{W}^h_V\mathbf{y}_i\bigg)\\
=~&\sum_{h=1}^H\mathbf{W}^h_V\frac{
\sum_{i=1}^pE_h(\mathbf{y}_t, \mathbf{x}'_i)\mathbf{x}'_i+\sum_{i=1}^{t}E_h(\mathbf{y}_t, \mathbf{y}_i)\mathbf{y}_i}{\sum_{i=1}^{p}E_h(\mathbf{y}_t, \mathbf{x}'_i)+\sum_{i=1}^{t}E_h(\mathbf{y}_t, \mathbf{y}_i)}\\
\end{aligned}
\]
for $t = 0, \dots, k-1$ where
\begin{equation}
\label{E_h}
E_h(\mathbf{x}, \mathbf{y}) = \exp(\frac{T}{\sqrt{d_\text{hidden}}}(\mathbf{W}^h_Q\mathbf{x})^\top(\mathbf{W}^h_K\mathbf{y})),
\end{equation}
and we set $\mathbf{y}_0=\mathbf{x}'_p$. When $t=0$, the second summation in both the numerator and denominator is interpreted as 0. The corresponding attention expression appearing in \eqref{attn_goal} is
\begin{equation}
\label{attn_goal_attn}
\begin{aligned}
\sum_{h=1}^H\mathbf{W}^h_V\frac{\sum_{j=1}^{m}E_h(\mathbf{y}'_t, \mathbf{s}_j)\mathbf{s}_j+\sum_{i=1}^{n}E_h(\mathbf{y}'_t, \mathbf{x}_i)\mathbf{x}_i +\sum_{i=1}^{t}E_h(\mathbf{y}'_t, \mathbf{y}_i)\mathbf{y}_i}{\sum_{j=1}^{m}E_h(\mathbf{y}'_t, \mathbf{s}_j)+\sum_{i=1}^{n}E_h(\mathbf{y}'_t, \mathbf{x}_i)+\sum_{i=1}^{t}E_h(\mathbf{y}'_t, \mathbf{y}_i)} = \sum_{h=1}^H\mathbf{W}^h_V\frac{A_h(\mathbf{x}, \mathbf{y}, \mathbf{s})}{B_h(\mathbf{x}, \mathbf{y}, \mathbf{s})}
\end{aligned}
\end{equation}
where $A_h$ and $B_h$ are used to represent the numerator and denominator respectively, with $\mathbf{y}'_0=\mathbf{x}_n, \mathbf{y}'_t=\mathbf{y}_t$ for all $t\ge1$. As discussed above, our objective is to enforce equality between the quantity in \eqref{attn_goal_attn} and the corresponding attention output $\mathbf{a}_{p+t}$ obtained from the source run.

To simplify the analysis, we first eliminate the summation over attention heads by allowing the prefix vectors $\mathbf{s}_j$ to be head-dependent. In addition, we remove the multiplication by the value matrix $\mathbf{W}^h_V$ by imposing the following, technically stronger, sufficient condition:
\begin{equation}
\label{suffice}
\frac{A_h(\mathbf{x}, \mathbf{y}, \mathbf{s})}{B_h(\mathbf{x}, \mathbf{y}, \mathbf{s})}
~=~ \frac{\sum_{i=1}^{p}E_h(\mathbf{y}_t, \mathbf{x}'_i)\mathbf{x}'_i+\sum_{i=1}^{t}E_h(\mathbf{y}_t, \mathbf{y}_i)\mathbf{y}_i}{\sum_{i=1}^{p}E_h(\mathbf{y}_t, \mathbf{x}'_i)+\sum_{i=1}^{t}E_h(\mathbf{y}_t, \mathbf{y}_i)}.
 \end{equation}

We further simplify the quotients appearing on the left-hand side by adopting an override strategy. Specifically, once a candidate prefix vector $\mathbf{s}_j$ is identified, it can be repeated arbitrarily many times in the prefix so that its contribution dominates the corresponding numerator and denominator in \eqref{suffice}. Consequently, it suffices to find $\mathbf{s}_j$ such that, for all $t = 0, \dots, k-1$,
\begin{equation}
\label{rest_s}
\frac{\sum_{j=1}^{m}E_h(\mathbf{y}'_t, \mathbf{s}_j)\mathbf{s}_j}{\sum_{j=1}^{m}E_h(\mathbf{y}'_t, \mathbf{s}_j)} =~\frac{\sum_{i=1}^{p}E_h(\mathbf{y}_t, \mathbf{x}'_i)\mathbf{x}'_i+\sum_{i=1}^{t}E_h(\mathbf{y}_t, \mathbf{y}_i)\mathbf{y}_i}{\sum_{i=1}^{p}E_h(\mathbf{y}_t, \mathbf{x}'_i)+\sum_{i=1}^{t}E_h(\mathbf{y}_t, \mathbf{y}_i)}
\end{equation}
holds. Once such vectors $\textbf{s}_j$ are obtained, repeated insertion allows \eqref{suffice} to be satisfied up to an arbitrarily small error. Since $\mathcal{T}$ is continuous, this approximation is sufficient for our purposes.

Notably, the embedding vectors $\mathbf{x}_i$ of the user input do not appear in \eqref{rest_s}. As a result, the resolution of \eqref{rest_s} for different attention heads is completely decoupled and can be carried out in parallel.

\subsection{Converting into Linear Equations}
\label{linear}
Solving \eqref{rest_s} constitutes the central technical challenge in our work. To address this, we apply the split-prefix assumption and seek $\mathbf{s}$ and $\Tilde{\mathbf{s}}$ so that
\begin{equation}
\label{rest_s_alt}
\frac{\sum_{j=1}^{m}E_h(\mathbf{y}'_t, \mathbf{s}_j)\Tilde{\mathbf{s}}_j}{\sum_{j=1}^{m}E_h(\mathbf{y}'_t, \mathbf{s}_j)}
=~ \frac{\sum_{i=1}^{p}E_h(\mathbf{y}_t, \mathbf{x}'_i)\mathbf{x}'_i+\sum_{i=1}^{t}E_h(\mathbf{y}_t, \mathbf{y}_i)\mathbf{y}_i}{\sum_{i=1}^{p}E_h(\mathbf{y}_t, \mathbf{x}'_i)+\sum_{i=1}^{t}E_h(\mathbf{y}_t, \mathbf{y}_i)}
\end{equation}
Since the target attention outputs are linear combinations of the source embeddings of $\{\mathbf{x}'_1, \dots, \mathbf{x}'_{p}\}$ and the generated outputs $\{\mathbf{y}_1, \dots, \mathbf{y}_{k-1}\}$, it is natural to construct each $\Tilde{\mathbf{s}}_j$ as a linear combination of these vectors and attempt to match their coefficients on both sides. Concretely, for each $j$ we assume
\[
\Tilde{\mathbf{s}}_j = \sum_{i=1}^{p}\lambda_{j, i}\mathbf{x}'_i + \sum_{i=1}^{k-1}\mu_{j, i}\mathbf{y}_i
\]
so that each $\Tilde{\mathbf{s}}_j$ introduces $p+k-1$ variables. With $m$ prefix vectors in total, this introduces $m(p+k-1)$ variables. It now suffices to solve the set of equations $SV=[X|Y]$ where we have the $k\times m$ coefficient matrix
\[
S = \begin{bmatrix}
\frac{E_h(\mathbf{x}_n, \mathbf{s}_1)}{\sum_{j=1}^{m}E_h(\mathbf{x}_n, \mathbf{s}_j)} & \cdots & \frac{E_h(\mathbf{x}_n, \mathbf{s}_m)}{\sum_{j=1}^{m}E_h(\mathbf{x}_n, \mathbf{s}_j)} \\

\vdots & \ddots & \vdots \\

\frac{E_h(\mathbf{y}_{k-1}, \mathbf{s}_1)}{\sum_{j=1}^{m}E_h(\mathbf{y}_{k-1}, \mathbf{s}_j)} & \cdots & \frac{E_h(\mathbf{y}_{k-1}, \mathbf{s}_m)}{\sum_{j=1}^{m}E_h(\mathbf{y}_{k-1}, \mathbf{s}_j)}
\end{bmatrix},
\]
$m\times(p+k-1)$ variable matrix
\[
V = \begin{bmatrix}
\lambda_{1, 1} & \cdots & \lambda_{1, p} & \mu_{1, 1} & \cdots & \mu_{1, k-1} \\
\vdots  & \ddots & \vdots & \vdots & \ddots & \vdots  \\
\lambda_{m, 1}  & \cdots & \lambda_{m, p} & \mu_{m, 1}  & \cdots & \mu_{m, k-1}
\end{bmatrix},
\]
$k\times p$ target matrix $X=(X_{t, \ell})_{0\le t\le k-1, 1\le \ell\le p}$, $k\times(k-1)$ target matrix $Y=(Y_{t, \ell})_{0\le t\le k-1, 1\le \ell\le k-1}$, in which their entries are given by
\[
X_{t, \ell} = \frac{E_h(\mathbf{y}_t, \mathbf{x}'_\ell)}{\sum_{i=1}^pE_h(\mathbf{y}_t, \mathbf{x}'_i)+\sum_{i=1}^{t}E_h(\mathbf{y}_t, \mathbf{y}_i)}, Y_{t, \ell} = \frac{E_h(\mathbf{y}_t, \mathbf{y}_j)}{\sum_{i=1}^pE_h(\mathbf{y}_t, \mathbf{x}'_i)+\sum_{i=1}^{t}E_h(\mathbf{y}_t, \mathbf{y}_i)},
\]
with the extra condition that $Y_{t, \ell} = 0$ if $t<\ell$.
This corresponds to $p+k-1$ linear systems of $m$ variables and $k$ equations each, sharing the same $k\times m$ coefficient matrix. Compared to the original nonlinear formulation, the dimensionality depends linearly on sequence length rather than embedding dimension, making the system tractable.

Intuitively, when $m$ is large enough, $\mathbf{s}_j$ can be chosen so that $S$ has rank $k$. Consequently, these $p+k-1$ linear systems admit a solution regardless of what is in $X$ and $Y$. There are some obvious counterexamples though: if $\mathbf{y}_i=\mathbf{y}_j$ for some $i, j$, then we will always have two identical rows in $S$ and it will never reach rank $k$. This turns out to be the only obstruction: If we assume that there exists $\mathbf{s}$ such that
\begin{equation}
\label{Vandermonde_assumption}
E_h(\mathbf{x}_n, \mathbf{s}), E_h(\mathbf{y}_1, \mathbf{s}), \dots, E_h(\mathbf{y}_{k-1}, \mathbf{s})\text{ are all distinct,}
\end{equation}
then we can always find $\mathbf{s}_j$ such that $S$ is rank $k$. The proof will be provided in Appendix \ref{proof_main}.

In our experiments, condition \eqref{Vandermonde_assumption} is never violated, even when there are repeated tokens in $\mathtt{y}$ or $\mathtt{x}_n$. This is due to the presence of positional encodings: even if two tokens are identical, they receive slightly different embeddings before entering $\mathcal{T}$.


\section{Algorithms}
\label{sec:alg}
We summarize our procedure below. Given a fixed Transformer $\mathcal{T}$, a source input $(\mathtt{x}'_1, \dots, \mathtt{x}'_p)$, a desired target output $(\mathtt{y}_1, \dots, \mathtt{y}_k)$, and the hyperparameters $m$ (the prefix length, which requires $m\ge k$) and $r$ (the repetition factor), the algorithm proceeds as follows:
\begin{enumerate}
    \item Run inference on the source input and collect all intermediate representations required by the attention matching procedure.
    \item For each attention head, construct $S$ and $\mathbf{s}_j$.
    \item For each attention head, solve the linear system $SV=[X|Y]$ to obtain the corresponding vectors $\Tilde{\mathbf{s}}_j$.
    \item Repeat each prefix vector $r$ times.
\end{enumerate}

The choice of $\mathbf{s}_j$ in step 2 requires careful consideration. In the proof of our main theorem, we provide a construction that guarantees $S$ has full rank. In practice, this construction performs poorly numerically. Below, we highlight a desirable numerical property of $\mathbf{s}_j$ and present a new construction that achieves this property. Empirically, this construction always satisfies the full-rank requirement. 


A key consideration when constructing $S$ is the override strategy underlying the derivation of the sufficient condition~\eqref{rest_s}. In theory, repeating the prefix vectors sufficiently many times guarantees successful override. In practice, however, this assumption may fail: due to the exponential scaling in the softmax operation, attention scores with large magnitudes can dominate, making them difficult to override using a limited number of prefix vectors. 

Consider the pre-exponential score (logit)
\[
l_h(\mathbf{x},\mathbf{y})=\frac{T}{\sqrt{d_\text{hidden}}}(\mathbf{W}^h_Q\mathbf{x})^\top(\mathbf{W}^h_K\mathbf{y}).
\]
The exponential mapping in attention makes the mechanism highly sensitive to differences in the scaled logits. As illustrated in Figure~\ref{fig:attn_score}, the maximum scaled attention logit $l_h$ can reach values as large as 20, in which case the dominant term ($\exp(l_h)\approx e^{20}$) overwhelms all smaller contributions, making random overrides practically infeasible. 

\begin{figure}
    \centering
    \includegraphics[width=0.8\linewidth, height=0.25\linewidth]{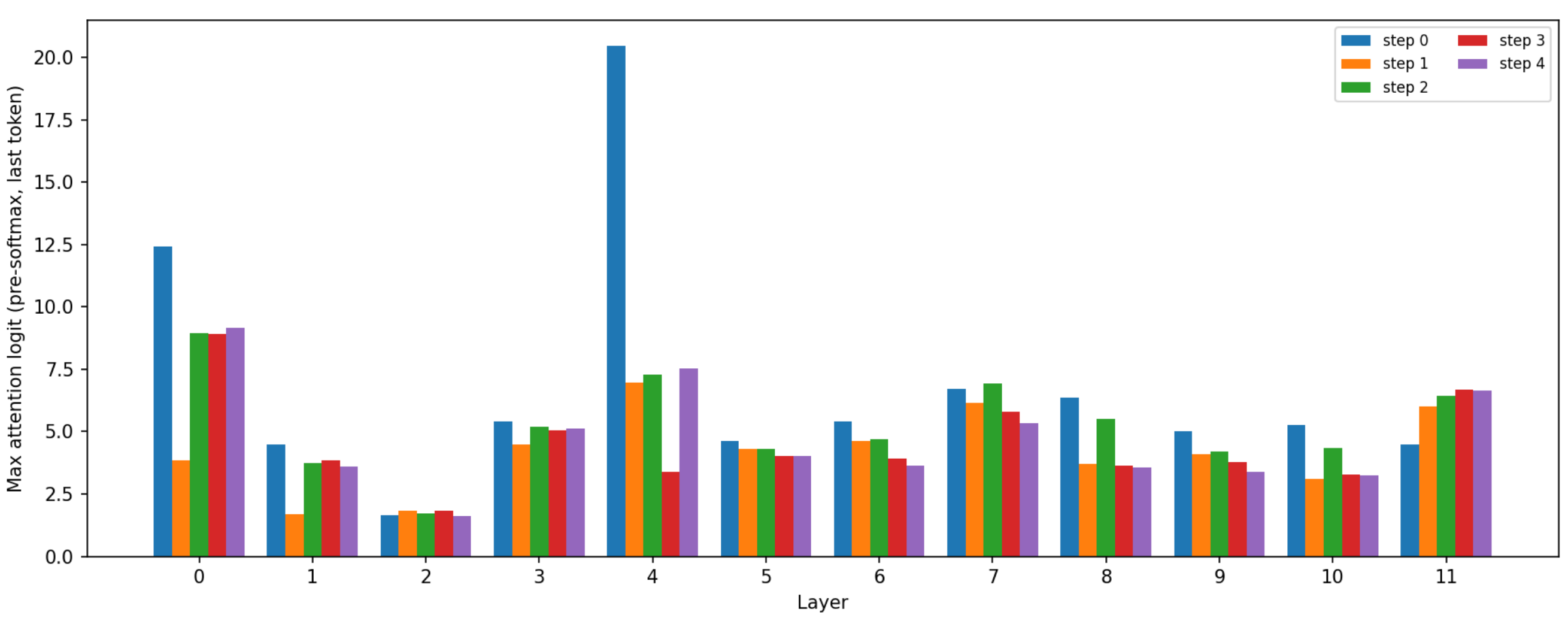}
    \caption{Maximum (across attention heads) pre-softmax scaled dot-product attention logits $l_h$ across layers and decoding steps in GPT-2. The model is prompted with “To be, or not to be, that is the question” and generates 5 tokens. Notably, for certain layer–step pairs, this maximum exceeds 20.}
    \label{fig:attn_score}
\end{figure}

To enable reliable override, it is therefore necessary not only to repeat the prefix multiple times but also to ensure that, for each decoding step $t$, there exists at least one prefix vector $\textbf{s}_j$ that produces a sufficiently large attention logit $l_h$. We formalize this requirement through the following override constraint. Given a threshold $attn_{lim}$,
\begin{equation}
\label{override_constraint}
\forall t \in \{0, 1, \ldots, k-1\}, \ \exists j \in \{1, 2, \ldots, m\} \ \text{s.t. } 
l_h(\mathbf{y}_t', \mathbf{s}_j) \ge attn_{lim}.
\end{equation}
We next propose a simple and constructive method for generating prefix vectors that satisfy this constraint. 
To induce a large attention score at step $t$, we construct a dedicated prefix vector $\mathbf{s}_{t+1}$ aligned with the direction of $(\mathbf{W}_K^h)^\top \mathbf{W}_Q^h \mathbf{y}_t'$. This choice is motivated by the fact that the scaled attention logit is linear in the key vector, and alignment with this direction maximizes the resulting attention score under a fixed norm.

Given this direction, the attention logit increases monotonically with the scaling factor.
We therefore select the scaling factor $\omega_{t+1}$ as the minimal value that satisfies the override constraint in Equation~\eqref{override_constraint}, i.e., such that $l_h(\mathbf{y}_t', \mathbf{s}_{t+1}) = attn_{lim}$. This choice yields a minimal-norm prefix vector that is sufficient to induce a large attention logit at the desired decoding step. Specifically, we choose
$\mathbf{s}_{t+1}
= \omega_{t+1} (\mathbf{W}_K^h)^\top \mathbf{W}_Q^h \mathbf{y}_t'$, where the scaling factor $\omega_{t+1}$ is determined accordingly.

Combining this construction with the override constraint in Equation~\eqref{override_constraint}, we obtain the following explicit formulation: Set $m=k$, and for $t = 0, \ldots, k-1$,
\begin{equation}
\label{sj_construction}
\mathbf{s}_{t+1}
= \frac{attn_{lim}}{T\|(\mathbf{W}_K^h)^\top \mathbf{W}_Q^h \mathbf{y}_t'\|^2}
(\mathbf{W}_K^h)^\top \mathbf{W}_Q^h \mathbf{y}_t'.
\end{equation}
%
\section{Experiment on Complete Control of Output Tokens on GPT-2}
\label{exper1}
In this experiment, we evaluate the ability of our algorithm to exert complete control over the generated output tokens across a wide range of source inputs, target outputs, and user inputs. All experiments are conducted on a workstation equipped with 4 consumer-grade GPUs and an Intel Core i9-9980XE CPU, and the full experimental sweep completes in approximately 2 hours. 

To satisfy the assumptions stated in Theorem~\ref{main_theorem}, we use Gemini Flash to construct four datasets that differ only in their terminal tokens, namely “but”, “and”, “.”, and “,”. Each dataset is further partitioned into buckets according to sentence length, with each bucket containing exactly 50 sentences. For all experiments, we set $attn_{lim}=20.0$, which was chosen based on preliminary experiments and to account for the relatively large attention scores produced by GPT-2, as discussed above. The prefix repetition count was fixed at 90 to align with the context window of GPT-2, allowing the repeated prefix to occupy a substantial portion of the available context without exceeding the model’s maximum input length.


For a fixed sentence length and terminal token, we consider all pairs of sentences within the same bucket. For each pair, one sentence is treated as the source input and the other as the user input. The success rate is defined as the number of pairs for which the computed prefix is successful, divided by the number of pairs in which the original generated outputs of the two sentences differ. In addition, we systematically vary the output length, the attention limit parameter $attn_{lim}$, and the number of prefix repetitions to assess the algorithm’s ability to generate output of different length and to identify the parameter settings that yield the highest success rate.

Table~\ref{tab:control_results} illustrates the results of the experiments. Overall, our method achieves a very high success rate. For short input/output lengths, the success rate is close to 1. Even for relatively longer input/output settings, the success rate remains in the range of 0.54–0.78, demonstrating strong robustness.

Further analysis of input length and output length shows that the success rate is primarily affected by the output length, while the impact of input length is comparatively minor. This behavior is expected, since the dimension of the constructed linear system is determined by output length; larger output length leads to higher-dimensional systems and increases the probability of an ill-conditioned linear system, thereby reducing the success rate. More investigations on the effects of the attention limit $attn_{lim}$ and the prefix repetition count are illustrated in Appendix~\ref{override_appendix}. 


\begin{table}[t]
\centering
\small
\setlength{\tabcolsep}{4pt}
\caption{Control success rates across datasets with different ending tokens, varying input length and output length. InLen and OutLen denote input length and output length, respectively.}
\begin{tabular}{llccccc}
\toprule
\textbf{Dataset} & \textbf{OutLen} 
& \multicolumn{5}{c}{\textbf{InLen}} \\
\cmidrule(lr){3-7}
 &  & 5 & 10 & 20 & 50 & 100 \\
\midrule
\multirow{3}{*}{\texttt{and}}
 & 1  & 1.0000 & 0.9931 & 0.9865 & 0.9804 & 0.9653 \\
 & 5  & 0.9559 & 0.9857 & 0.9220 & 0.9127 & 0.7812 \\
 & 10 & 0.8110 & 0.9722 & 0.9127 & 0.7857 & 0.6710 \\
\midrule
\multirow{3}{*}{\texttt{but}}
 & 1  & 1.0000 & 1.0000 & 1.0000 & 0.9992 & 0.9865 \\
 & 5  & 1.0000 & 1.0000 & 0.8702 & 0.9355 & 0.9282 \\
 & 10 & 0.9527 & 0.9506 & 0.9122 & 0.7739 & 0.7641 \\
\midrule
\multirow{3}{*}{\texttt{comma}}
 & 1  & 1.0000 & 1.0000 & 0.9988 & 0.8959 & 0.9327 \\
 & 5  & 0.9620 & 0.9682 & 1.0000 & 0.8155 & 0.8224 \\
 & 10 & 0.9171 & 0.9139 & 0.8433 & 0.6282 & 0.5988 \\
\midrule
\multirow{3}{*}{\texttt{dot}}
 & 1  & 1.0000 & 1.0000 & 1.0000 & 0.9029 & 0.9943 \\
 & 5  & 0.9918 & 0.9053 & 0.8555 & 0.7816 & 0.8861 \\
 & 10 & 0.8196 & 0.6649 & 0.7551 & 0.5461 & 0.6743 \\
\bottomrule
\end{tabular}
\label{tab:control_results}
\end{table}

\section{Extension to Grouped-Query Attention and Qwen3-4B} 
\label{sec:gqa}

We conduct additional experiments on Qwen3-4B \citep{qwen3}, whose grouped-query attention (GQA)~\citep{GQA} architecture falls outside the scope of our theoretical guarantees. Nevertheless, our algorithm can be readily adapted to accommodate GQA. In GQA, used by many current LLMs~\citep{qwen3, gpt-oss}, each group of $G$ query heads shares one key/value head. Since a prefix enters the attention only through the key and value matrices, the query heads $h_1,\dots,h_G$ of a group must also share it, and we stack their linear systems of Section~\ref{linear} with a common variable matrix $V$:
\begin{equation}
\label{gqa_stacked}
\begin{bmatrix} S^{(1)}\\ \vdots\\ S^{(G)} \end{bmatrix} V
= \begin{bmatrix} [X^{(1)}|Y^{(1)}]\\ \vdots\\ [X^{(G)}|Y^{(G)}] \end{bmatrix}.
\end{equation}
Each key/value head thus receives $m=Gk$ prefix vectors, one per query head $h_i$ and decoding step $t$, built from the query of $h_i$ at step $t$ as $\mathbf{s}_{t+1}$ in~\eqref{sj_construction}; $G=1$ recovers $m=k$. Proposition~\ref{prop_linear_system} does not cover the stacked $Gk\times Gk$ matrix $S$, which could in principle be singular or ill-conditioned. We have not observed this in our experiments.

\paragraph{QK-Norm.} Qwen3-4B also applies QK-Norm: every query and key head is normalized by an RMSNorm~\citep{zhang2019rmsnorm}, $\mathrm{RMSNorm}(\mathbf{k})=\sqrt{d_\text{hidden}}\,\boldsymbol{\gamma}\odot\mathbf{k}/\|\mathbf{k}\|_2$ with a learned gain $\boldsymbol{\gamma}$. Hence every key has the form $\boldsymbol{\gamma}\odot\mathbf{u}$ with $\|\mathbf{u}\|_2=\sqrt{d_\text{hidden}}$, and the scaling factor in~\eqref{sj_construction} has no effect. For the same reason, the construction in the proof of Proposition~\ref{prop_linear_system} does not apply. For a normalized query $\mathbf{q}$, Cauchy--Schwarz gives
\begin{equation}
\label{qknorm_bound}
\mathbf{q}^\top(\boldsymbol{\gamma}\odot\mathbf{u})/\sqrt{d_\text{hidden}}\le\|\boldsymbol{\gamma}\odot\mathbf{q}\|_2,
\quad\text{with equality at }\
\mathbf{u}^\star=\sqrt{d_\text{hidden}}\,\boldsymbol{\gamma}\odot\mathbf{q}/\|\boldsymbol{\gamma}\odot\mathbf{q}\|_2,
\end{equation}
so the dedicated key is $\boldsymbol{\gamma}\odot\mathbf{u}^\star$, and the attained logit $\|\boldsymbol{\gamma}\odot\mathbf{q}\|_2$ replaces $attn_{lim}$, leaving no free parameter. We verified numerically that this key is produced by an input vector of the attention module (i.e., $\mathbf{W}_K^h$ has full row rank, so that $\mathbf{W}_K^h\mathbf{s}=\mathbf{u}^\star$ has a solution $\mathbf{s}$), so it is a prefix in the sense of Section~\ref{sec:prefix}. By~\eqref{qknorm_bound}, each query attains the largest possible logit $\|\boldsymbol{\gamma}\odot\mathbf{q}\|_2$ on its own prefix vector. Divided by this value, the logits form a matrix with ones on the diagonal and, off the diagonal, the cosines between the vectors $\boldsymbol{\gamma}\odot\mathbf{q}$ of different queries, which are smaller than one unless two of these vectors point in the same direction; hence every row of $S$ has its largest entry on the diagonal. As this logit is a ceiling, a real token can occasionally outscore the dedicated prefix vector, which repetition compensates. Repetition leaves $S$ and $V$ unchanged and only raises the share of attention on the prefix, which tends to one as $r$ grows; the limit argument in Appendix~\ref{proof_main} requires no lower bound on the logits.

\paragraph{Results.} We repeat the experiment of Section~\ref{exper1} on Qwen3-4B, with the four datasets rebuilt for the Qwen3 tokenizer so that every bucket again holds 50 sentences ending in the common terminal token and no pair violates $\mathtt{x}_n=\mathtt{x}'_p$. As Table~\ref{table:qwen3} shows, the prefix forces the target output in $99.65\%$ of the instances without any repetition ($r=1$), and with $r=10$ all groups reach $100\%$. This is far above prompt-based steering: even a system prompt instructing the model to output the target succeeds only $84.6\%$ of the time (Appendix~\ref{app:baselines}). We conjecture that QK-Norm, by bounding the logits, weakens the softmax sensitivity discussed in Section~\ref{sec:alg} and keeps the linear system of Section~\ref{linear} well conditioned.

\begin{table}[t]
\centering
\small
\caption{Success rates on Qwen3-4B, averaged over the 15 (OutLen, InLen) pairs of each dataset. The datasets are those of Table~\ref{tab:control_results}, rebuilt for the Qwen3 tokenizer with 50 sentences per bucket.}
\label{table:qwen3}
\begin{tabular}{lcccc}
\toprule
& \texttt{and} & \texttt{but} & \texttt{comma} & \texttt{dot} \\
\midrule
$r = 1$  & 0.9977 & 0.9966 & 0.9950 & 0.9968 \\
$r = 10$ & 1.0000 & 1.0000 & 1.0000 & 1.0000 \\
\bottomrule
\end{tabular}
\end{table}

\section{Conclusions}
We settle a long-standing open problem regarding the expressive power of prefix-tuning in practical large language models and fine-tuning settings by providing and proving a complete characterization of the existence of a prefix capable of inducing essentially arbitrary output sequences. Unlike prior work, our result applies to modern LLM architectures with arbitrary weights, and is fully constructive: we provide an explicit algorithm for generating a prefix that produces any desired target token sequence. The construction extends to grouped-query attention, where it forces the target output of Qwen3-4B in every tested setting. This provides strong evidence for the practicality of prefix-tuning as a parameter-efficient fine-tuning method. 





\bibliographystyle{plainnat}
\bibliography{references} 

\appendix

\section{Comparison between Autoregressive Generation and Sequence-to-Sequence Generation}
\label{seq-to-seq}
A few prior works have shown that a Transformer with a prefix can approximate arbitrary sequence-to-sequence functions \citep{Wang2023UniversalityAL, PTB1, HuWGL0025}, which appears to be a much stronger result than ours (which considers only a single input-output pair). This makes it tempting to derive our result as a consequence of theirs. However, as noted in the main text, their proofs rely on specially crafted Transformers, whereas ours apply to any pretrained Transformer specified by the user, aligning more closely with practical fine-tuning scenarios. Beyond this distinction, this section highlights another subtle but important difference: autoregressive generation imposes significantly stronger requirements on prefix construction than sequence-to-sequence approximation or generation.

Essentially, we aim to find a prefix $\mathcal{S}$ (with slightly simplified and abused notation) such that
\[
\mathcal{T}_\mathcal{S}(\mathtt{x}) = \mathtt{y}_1, \mathcal{T}_\mathcal{S}(\mathtt{x}, \mathtt{y}_1) = \mathtt{y}_2, \cdots, \mathcal{T}_\mathcal{S}(\mathtt{x}, \mathtt{y}_1, \cdots, \mathtt{y}_{k-1})=\mathtt{y}_k
\]
(see \eqref{ioc_goal} for the precise formulation). If one attempts to apply sequence-to-sequence approximation results to construct $\mathcal{S}$, a choice must be made regarding which of these equations to target. For instance, one could apply existing expressiveness results to obtain a prefix $\mathcal{S}$ satisfying $\mathcal{T}_\mathcal{S}(\mathtt{x}) = \mathtt{y}_1$, or $\mathcal{T}_\mathcal{S}(\mathtt{x}, \mathtt{y}_1) = \mathtt{y}_2$, and so on. The key issue is that the resulting prefix $\mathcal{S}$ would generally differ across these choices. In contrast, our approach constructs a single, consistent prefix $\mathcal{S}$ that satisfies all $k$ equations simultaneously.

\section{Counterexamples}
\label{2-element}
\begin{example}[$\mathtt{x}_n=\mathtt{x}'_p$ is necessary]
Let $\mathcal{V}=\{a, b\}$ be a vocabulary with two tokens $a$ and $b$, and let $\textbf{Emb}(a)=(1, 0)$ and $\textbf{Emb}(b)=(0, 1)$ be their embeddings. Let $\mathcal{T}$ be a one-layer, one-head Transformer with a residual connection, where the weight matrix $W_V$ is the zero matrix. In other words, the attention module makes no contribution. Thus, $\mathcal{T}$ is essentially the identity map from $\mathbb{R}^2$ to $\mathbb{R}^2$ (for an input sequence, it acts as the identity on the last token).

At the output of $\mathcal{T}$, let the first coordinate correspond to the logit of $a$ and the second to that of $b$. Then, $\mathcal{T}$ simply takes the last observed token as the next-token prediction.

Now, in Theorem \ref{main_theorem}, let the target input and output be $(a)$ and $(b)$, respectively. The token $b$ is clearly reachable by $\mathcal{T}$ (e.g. with input $(b)$). However, no prefix $\mathcal{S}$ can satisfy $\mathcal{T}_{\mathcal{S}}((a))=b$, since the output is entirely determined by the last input token. The model will therefore output $a$ again.
\end{example}

\begin{example}[$p=n$ is necessary]
Consider the same setting as in Example 1, but now suppose we use a positional encoding such that every embedding from the second position onward is added by $(100, 0)$, forcing the model to always predict $a$ whenever the input sequence has length at least two.

Now, let the target input and output be $(a, b)$ and $(b)$, respectively. The target $(b)$ remains reachable from the source input $(b)$. Note that, unlike in Example 1, the last tokens of the target and source inputs already coincide. However, the next-token prediction for $(a, b)$, even with a prefix, will always be $a$ due to the positional encoding. This limitation arises fundamentally from the mismatch in sequence lengths, i.e., $p \neq n$.
\end{example}

\begin{example}[Two input-output pairs, Proof of Theorem \ref{2_data}]
Consider the following two input-output pairs:
\begin{equation}
\begin{aligned}
\mathtt{x} = (\mathtt{x}_1, \mathtt{x}_2), &\mathtt{y} = (\mathtt{y}_1, \mathtt{y}_2)\\
\mathtt{x} = (\mathtt{x}_1, \mathtt{x}_2, \mathtt{y_1}), &\mathtt{y} = (\mathtt{y}_3)\\
\end{aligned}
\end{equation}
where $\mathtt{y}_2 \neq \mathtt{y}_3$. If $\mathcal{T}$ satisfies the first pair, then when given the input from the second pair, it must generate $\mathtt{y}_2$ rather than $\mathtt{y}_3$. Therefore, these two pairs cannot be simultaneously realized by $\mathcal{T}$, even under full fine-tuning.
\end{example}

\begin{example}[Two input-output pairs, with extra assumptions]
It is tempting to address Example 3 by imposing stronger assumptions. Suppose we additionally require:
\begin{itemize}
\item The two input-output pairs are jointly realizable via full fine-tuning;
\item Each input-output pair, considered in isolation, is realizable by some prefix.
\end{itemize}
Can we then conclude that there exists a single prefix that realizes both pairs simultaneously? Unfortunately, the answer remains negative.

Consider again a one-layer, one-head Transformer $\mathcal{T}$, but now let $W_Q$ and $W_K$ be zero matrices and $W_V$ be the identity matrix, with no residual connections or positional encodings. In this case, all tokens attend uniformly (since the softmax is applied to a zero vector), and the output of $\mathcal{T}$ is simply the average of the token embeddings. Equivalently, next-token prediction reduces to a majority vote: the most frequent token is selected.

Now consider the two input-output pairs:
\begin{equation}
\begin{aligned}
\mathtt{x} = (a), &\quad \mathtt{y} = (b),\\
\mathtt{x} = (b), &\quad \mathtt{y} = (a).
\end{aligned}
\end{equation}
These can be realized via full fine-tuning (e.g. by setting $W_V$ to the negative identity). Individually, each pair can also be realized by a prefix $\mathcal{S}$: a prefix consisting of many $b$ tokens for the first pair, or many $a$ tokens for the second. However, no single prefix can satisfy both simultaneously, as it cannot contain strictly more $a$ than $b$ and strictly more $b$ than $a$ at the same time.

At a high level, Example 3 arises from contradictory requirements within the fine-tuning dataset. In contrast, in this example the contradiction stems not only from the dataset but also from the model’s aggregation behavior. Consequently, addressing Problem \ref{prob} for larger datasets likely requires consistency conditions on both the dataset and the model. Formulating such conditions is nontrivial, and we hope this work encourages further progress in that direction.
\end{example}



\section{Attack on LLM memory}
\label{exper2}
In \citep{PTB1}, the authors postulate that the expressiveness of prefix-tuning may pose security risks to LLM systems, although they do not establish whether such attacks can manifest in real-world settings. In this section, we provide a proof-of-concept attack scenario based on our algorithm. This attack is not intended to compromise real-world deployments, particularly when reasonable defenses are in place; rather, our goal is to move one step beyond the speculative concerns raised in \citep{PTB1}.

Recall that our algorithm enables the model to produce a predetermined response with high probability, largely independent of the user input, provided that we have access to the model and the ability to replace its KV cache with one augmented by a malicious prefix. This capability relies on two key assumptions: first, that there exists a source input capable of generating the desired target output, and second, that the user input has the same length as the source input. In this experiment, we demonstrate a method for identifying a suitable source input for a deliberately constructed hazardous target output, and we further show that the attack remains effective even when the user input length does not exactly match that of the source input. 

For GPT-2, we identify two effective approaches for generating such source inputs. The first approach is to directly prompt the model to repeat a given sentence, since GPT-2 lacks post-training mechanisms to avoid hazardous outputs. The second approach is to repeat the target sentence sufficiently many times in the prompt, which causes GPT-2 to continue generating the same sentence. For state-of-the-art models equipped with post-training safety alignment to avoid hazardous outputs, we refer readers to the literature on jailbreak attacks \citep{xu2024comprehensive, yi2024jailbreakattacksdefenseslarge}.

As a concrete example, to generate the humiliating sentence “You are an idiot.”, one can simply repeat this sentence multiple times in the input, after which GPT-2 will reproduce it in the output. Since the resulting source input may not have the same length as the user input, we additionally evaluate the robustness of our algorithm by conducting experiments in which the same source input is used while varying the length of the user input.

\begin{table}[t]
\centering
\caption{Success rate under different input lengths.
Source Input Length = 50 (\textit{``You are an idiot.''} $\times$ 10),
Output Length = 10 (\textit{``You are an idiot.''} $\times$ 2).}
\label{tab:misalign_inlen}
\begin{tabular}{c c c c c c}
\hline
\textbf{InLen} & 5 & 10 & 20 & 50 & 100 \\
\hline
\textbf{Success Rate} & 1.00 & 1.00 & 1.00 & 1.00 & 0.92 \\
\hline
\end{tabular}
\end{table}

Despite the theoretical requirement that the source input length should match the user input length, the experimental results demonstrate surprising robustness to variations in user input length. As shown in Table~\ref{tab:misalign_inlen}, the success rate remains at 100\% across a wide range of input lengths, and only slightly drops when the input length deviates substantially from the source length. One possible explanation is that positional embeddings contribute only a small fraction of the final token representation, making the attack relatively insensitive to moderate length mismatches.

On the other hand, our experiments with ending token mismatches consistently fail, indicating that matching the final token is crucial in practice.

\section{Proof of the Main Theorem}
\label{proof_main}
Most of the content of the proof of \ref{main_theorem} is already included in the main text, except the solubility of the linear system $SV=[X|Y]$ when \eqref{Vandermonde_assumption} is satisfied. We formulate this fact precisely and provide the proof below.
\begin{proposition}
\label{prop_linear_system}
If \eqref{Vandermonde_assumption} is satisfied, then there exists $\mathbf{s}_1, \dots, \mathbf{s}_k$ such that the $k\times k$ matrix
\[
S = \begin{bmatrix}
\frac{E_h(\mathbf{x}_n, \mathbf{s}_1)}{\sum_{j=1}^{k}E_h(\mathbf{x}_n, \mathbf{s}_j)} & \cdots & \frac{E_h(\mathbf{x}_n, \mathbf{s}_k)}{\sum_{j=1}^{k}E_h(\mathbf{x}_n, \mathbf{s}_j)} \\

\vdots & \ddots & \vdots \\

\frac{E_h(\mathbf{y}_{k-1}, \mathbf{s}_1)}{\sum_{j=1}^{k}E_h(\mathbf{y}_{k-1}, \mathbf{s}_j)} & \cdots & \frac{E_h(\mathbf{y}_{k-1}, \mathbf{s}_k)}{\sum_{j=1}^{k}E_h(\mathbf{y}_{k-1}, \mathbf{s}_j)}
\end{bmatrix}
\]
is non-singular (note that we have set $m=k$).
\end{proposition}
\begin{proof}
Since multiplying each row by a constant preserves non-singularity, it suffices to find $\mathbf{s}_1, \dots, \mathbf{s}_k$ such that the $k\times k$ matrix
\[
\overline{S} = \begin{bmatrix}
E_h(\mathbf{x}_n, \mathbf{s}_1) & \cdots & E_h(\mathbf{x}_n, \mathbf{s}_k) \\

\vdots & \ddots & \vdots \\

E_h(\mathbf{y}_{k-1}, \mathbf{s}_1) & \cdots & E_h(\mathbf{y}_{k-1}, \mathbf{s}_k)
\end{bmatrix}
\]
is non-singular. By our assumption \eqref{Vandermonde_assumption}, there exists $\mathbf{s}$ such that
\[
E_h(\mathbf{x}_n, \mathbf{s}), E_h(\mathbf{y}_1, \mathbf{s}), \dots, E_h(\mathbf{y}_{k-1}, \mathbf{s})\text{ are all distinct.}
\]
Consequently, we can simply take
\[
\mathbf{s}_j = (j-1)\mathbf{s}.
\]
Then, we have
\[
\begin{aligned}
E_h(\mathbf{x}_n, \mathbf{s}_j) &= \exp(\frac{T}{\sqrt{d_\text{hidden}}}(\mathbf{W}^h_Q\mathbf{x}_n)^\top((j-1)\mathbf{W}^h_K\mathbf{s}))=\bigg(\exp(\frac{T}{\sqrt{d_\text{hidden}}}(\mathbf{W}^h_Q\mathbf{x}_n)^\top(\mathbf{W}^h_K\mathbf{s}))\bigg)^{j-1}\\
&= E_h(\mathbf{x}_n, \mathbf{s})^{j-1}.
\end{aligned}
\]
Similar equalities hold for $\mathbf{y}_1, \dots, \mathbf{y}_{k-1}$. Therefore, $\overline{S}$ is a Vandermonde matrix and it is well-known that its determinant is given by the product of mutual differences of $E_h(\mathbf{x}_n, \mathbf{s}), E_h(\mathbf{y}_1, \mathbf{s}), \dots, E_h(\mathbf{y}_{k-1}, \mathbf{s})$, which is non-zero by our assumption.
\end{proof}
\begin{proof}[Proof of Theorem \ref{main_theorem}]
By Proposition \ref{prop_linear_system}, the linear system $SV = [X|Y]$ is solvable at every attention head. This gives $\mathbf{s}_j, \Tilde{\mathbf{s}}_j$ ($j=1, \dots, k$) such that \eqref{rest_s_alt} holds. To establish \eqref{rest_s} from \eqref{rest_s_alt}, note that if all $\mathbf{s}_j, \Tilde{\mathbf{s}}_j$ are repeated $N$ times, the quotient $A_h(\mathbf{x}, \mathbf{y}, \mathbf{s}) / B_h(\mathbf{x}, \mathbf{y}, \mathbf{s})$ becomes
\[
\begin{aligned}
&\frac{N\sum_{j=1}^{m}E_h(\mathbf{y}'_t, \mathbf{s}_j)\Tilde{\mathbf{s}}_j+\sum_{i=1}^{n}E_h(\mathbf{y}'_t, \mathbf{x}_i)\mathbf{x}_i+\sum_{i=1}^{t}E_h(\mathbf{y}'_t, \mathbf{y}_i)\mathbf{y}_i}{N\sum_{j=1}^{m}E_h(\mathbf{y}'_t, \mathbf{s}_j)+\sum_{i=1}^{n}E_h(\mathbf{y}'_t, \mathbf{x}_i)+\sum_{i=1}^{t}E_h(\mathbf{y}'_t, \mathbf{y}_i)}\\
=~& \frac{\sum_{j=1}^{m}E_h(\mathbf{y}'_t, \mathbf{s}_j)\Tilde{\mathbf{s}}_j+\frac{\sum_{i=1}^{n}E_h(\mathbf{y}'_t, \mathbf{x}_i)\mathbf{x}_i+\sum_{i=1}^{t}E_h(\mathbf{y}'_t, \mathbf{y}_i)\mathbf{y}_i}{N}}{\sum_{j=1}^{m}E_h(\mathbf{y}'_t, \mathbf{s}_j)+\frac{\sum_{i=1}^{n}E_h(\mathbf{y}'_t, \mathbf{x}_i)+\sum_{i=1}^{t}E_h(\mathbf{y}'_t, \mathbf{y}_i)}{N}}.
\end{aligned}
\]
For a given $\epsilon>0$ there exists a positive integer $N_h$ such that
\[
\bigg|\bigg|\frac{\sum_{j=1}^{m}E_h(\mathbf{y}'_t, \mathbf{s}_j)\Tilde{\mathbf{s}}_j}{\sum_{j=1}^{m}E_h(\mathbf{y}'_t, \mathbf{s}_j)} - \frac{\sum_{j=1}^{m}E_h(\mathbf{y}'_t, \mathbf{s}_j)\Tilde{\mathbf{s}}_j+\frac{\sum_{i=1}^{n}E_h(\mathbf{y}'_t, \mathbf{x}_i)\mathbf{x}_i+\sum_{i=1}^{t}E_h(\mathbf{y}'_t, \mathbf{y}_i)\mathbf{y}_i}{N}}{\sum_{j=1}^{m}E_h(\mathbf{y}'_t, \mathbf{s}_j)+\frac{\sum_{i=1}^{n}E_h(\mathbf{y}'_t, \mathbf{x}_i)+\sum_{i=1}^{t}E_h(\mathbf{y}'_t, \mathbf{y}_i)}{N}}\bigg|\bigg|_\infty < \epsilon
\]
for all $N\geq N_h$. By taking the maximum over all attention heads and all $t$, we have shown that for any given $\epsilon>0$, there exists a generalized prefix $\mathcal{S}$ such that for every multi-head attention module $\mathcal{A}$, we have
\[
||\mathcal{A}_\mathcal{S}(\mathbf{y}'_t) - \mathcal{A}(\mathbf{y}_t)||_\infty < Hd_{in}\bigg(\max_{h}||\mathbf{W}^h_V||_\infty\bigg)\epsilon
\]
for $0\le t\le k-1$. Here the constant $H$ comes from aggregating all attention heads, $d_{in}\max_{h}||\mathbf{W}^h_V||_\infty$ comes from the multiplication by $\mathbf{W}^h_V$, and $\mathcal{A}_\mathcal{S}$ denotes the attention module after applying the prefix $\mathcal{S}$ (for simplicity, we omit the contribution of previous tokens).


If $\mathcal{T}$ does not have residual connections, the proof is essentially complete as we have guaranteed that the outputs of all multi-head attention modules with the user input $\mathtt{x}$ can be arbitrarily close to those with the source input $\mathtt{x}'$. It is also clear from the proof we only require prefixes at the last layer. If residual connections are present, an input embedding is transformed as
\[
\mathbf{x}'\longrightarrow\mathbf{x}' + \mathcal{A}(\mathbf{x}').
\]
Our attention matching has already achieved
\[
\mathcal{A}_\mathcal{S}(\mathbf{y}') \approx\mathcal{A}(\mathbf{y})
\]
up to a factor of $\epsilon$. To do the same for the residual outputs, we further require the condition $\textbf{Emb}(\mathtt{x}'_p) = \textbf{Emb}(\mathtt{x}_n)$. This clearly holds when $p=n$ and $\mathtt{x}'_p=\mathtt{x}_n$.

\end{proof}

\section{Considering Biases in the Attention}
\label{appendix_bias}
Previously, we ignored the bias in the attention mechanism for simplicity. In this section, we show that our theory and algorithm still hold if we introduce biases in $Q$, $K$, and $V$. The $Q$ and $K$ biases are handled jointly, whereas the $V$ bias is treated independently. We start with the $V$ cache.

We first introduce the bias to the $V$ term. In practical models, instead of directly using $\mathbf{W}_V^h\mathbf{x}_i'$ and $\mathbf{W}_V^h\mathbf{y}_i$, the terms become $\mathbf{W}_V^h\mathbf{x}_i' + \mathbf{bias}_V^h$ and $\mathbf{W}_V^h\mathbf{y}_i + \mathbf{bias}_V^h$, respectively.

This modification yields the following expression for the attention output:
\[
\sum_{h=1}^H\frac{A^h_{vbias}(\mathbf{x}, \mathbf{y}, \mathbf{s})}{B_h(\mathbf{x}, \mathbf{y}, \mathbf{s})},
\]
where the numerator is given by
\[
\begin{aligned}
A^h_{vbias}(\mathbf{x}, \mathbf{y}, \mathbf{s}) =& \sum_{j=1}^{m}E_h(\mathbf{y}'_t, \mathbf{s}_j)(\mathbf{W}_V^h\mathbf{s}_j+\mathbf{bias}_V^h)+\sum_{i=1}^{n}E_h(\mathbf{y}'_t, \mathbf{x}_i)(\mathbf{W}_V^h\mathbf{x}_i+\mathbf{bias}_V^h)\\
&+\sum_{i=1}^{t}E_h(\mathbf{y}'_t, \mathbf{y}_i)(\mathbf{W}_V^h\mathbf{y}_i+\mathbf{bias}_V^h)\\
=& \mathbf{W}_V^h A_h + \left(\sum_{j=1}^m E_h(\mathbf{y}_t', \mathbf{s}_j)+\sum_{i=1}^n E_h(\mathbf{y}_t',\mathbf{x}_i)+\sum_{i=1}^t E_h(\mathbf{y}_t',\mathbf{y}_i)\right)\mathbf{bias}_V^h \\
=& \mathbf{W}_V^h A_h + B_h\mathbf{bias}_V^h.
\end{aligned}
\]
Then the aggregated attention expression can be written as:
\[
\sum_{h=1}^H\frac{A^h_{vbias}}{B_h} = \sum_{h=1}^H \left( \mathbf{W}_V^h\frac{A^h}{B_h} + \mathbf{bias}_V^h \right) = \sum_{h=1}^H \mathbf{W}_V^h\frac{A^h}{B_h} + \sum_{h=1}^H\mathbf{bias}_V^h.
\]
The same logic applies to the RH Sof our sufficient condition~\eqref{suffice}, where the additional $V$ bias term naturally cancels out on both sides without requiring extra conditions.

Next, we address the $Q$ and $K$ caches. In the main text, $Q$ and $K$ were introduced in Section~\ref{attn_matching} within the $E_h$ term. Simply adding bias to the $E_h$ term does not alter the analysis in that section. However, the construction of the prefix in equation~\eqref{sj_construction} requires modification.

Since the attention logit becomes
\[
l_h(\mathbf{x},\mathbf{y})=\frac{T}{\sqrt{d_\text{hidden}}}(\mathbf{W}_Q^h\mathbf{x}+\mathbf{bias}_Q^h)^\top(\mathbf{W}_K^h\mathbf{y}+\mathbf{bias}_K^h),
\]
the direction of the prefix vector $\mathbf{s}_{t+1}$ must be adjusted to:
\[
\mathbf{s}_{t+1} = \omega_{t+1} (\mathbf{W}_K^h)^\top (\mathbf{W}_Q^h \mathbf{y}_t'+\mathbf{bias}_Q^h).
\]
The corresponding scaling factor must also account for the $K$-bias. This yields the final explicit formulation:
\[
\begin{aligned}
\text{Set } m = k. \quad \text{For } t = 0, \ldots, k-1: \\
\mathbf{s}_{t+1} = \frac{\sqrt{d_\text{hidden}}\text{attn}_{lim} - T(\mathbf{W}_Q^h \mathbf{y}_t'+\mathbf{bias}_Q^h)^\top\mathbf{bias}_K^h}{T\|(\mathbf{W}_K^h)^\top (\mathbf{W}_Q^h \mathbf{y}_t'+\mathbf{bias}_Q^h)\|^2} (\mathbf{W}_K^h)^\top (\mathbf{W}_Q^h \mathbf{y}_t'+\mathbf{bias}_Q^h).
\end{aligned}
\]

Finally, our proof of the main theorem in Section~\ref{proof_main} depends on the concrete form of $Q$ and $K$. The $Q$-bias does not influence the proof structure; simply adding the bias to the $Q$-term does not introduce conflicts. On the other hand, including the $K$-bias requires a brief explanation regarding the non-singularity of matrix $\overline{S}$. 

Recall that $E_h(\mathbf{x}, \mathbf{s}) = \exp\left( (\mathbf{W}_Q^h\mathbf{x}+\mathbf{bias}_Q^h)^\top (\mathbf{W}_K^h\mathbf{s}+\mathbf{bias}_K^h) \right)$. Expanding the exponent gives:
\[
(\mathbf{W}_Q^h\mathbf{x}+\mathbf{bias}_Q^h)^\top \mathbf{W}_K^h\mathbf{s} + (\mathbf{W}_Q^h\mathbf{x}+\mathbf{bias}_Q^h)^\top \mathbf{bias}_K^h.
\]
The second term depends only on the query $\mathbf{x}$ (the row index in $\overline{S}$) and not on the prefix $\mathbf{s}$ (the column index). Therefore, adding $K$-bias is equivalent to multiplying each row of the original $\overline{S}$ matrix by a non-zero constant factor:
\[
C(\mathbf{x}) = \exp\left( \frac{T}{\sqrt{d_\text{hidden}}} (\mathbf{W}_Q^h\mathbf{x}+\mathbf{bias}_Q^h)^\top \mathbf{bias}_K^h \right).
\]
Since multiplying rows by non-zero constants preserves linear independence, the modified $\overline{S}$ remains non-singular.

\section{Effect of the Override Parameters on GPT-2}
\label{override_appendix}

We analyze model performance under varying attention limits ($attn_{lim}$) and repetition counts (repeat) across different input-output length configurations, measured by success rate over 2,450 trials per setting. The results are illustrated in Figure~\ref{fig:override_parameter}.
\begin{figure}
    \centering
    \includegraphics[width=1.2\linewidth]{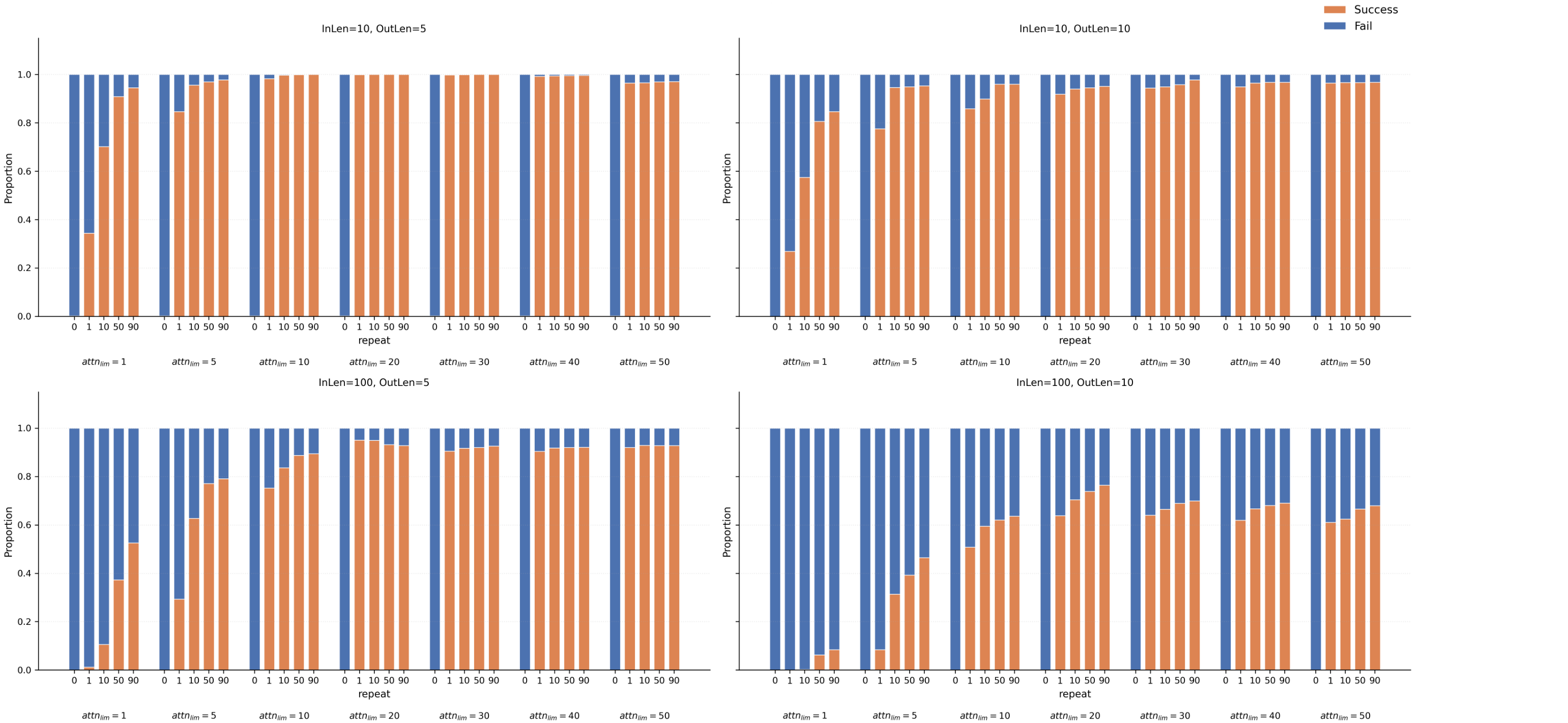}
    \caption{Performance Breakdown Across Attention Limits and Repetition Counts}
    \label{fig:override_parameter}
\end{figure}
Across all input--output length configurations, we observe consistent trends in how the attention limit ($attn_{lim}$) and repetition count ($\text{repeat}$) affect model performance.

When varying the attention limit while holding repetition fixed, the success rate increases sharply at small attention limits (e.g., from $attn_{lim}=1$ to $5$ or $10$), suggesting that insufficient attention limit is a primary cause of failure.
At very small attention limits, the attention scores are too weak to reliably preserve relevant information, such that even a large number of repetitions (e.g., $\text{repeat}=90$) still results in frequent override failures.
Beyond moderate values ($attn_{lim}\approx 20$--$30$), further increasing the attention limit yields diminishing returns, with performance largely saturating or exhibiting only marginal improvements.
One possible explanation is that excessively large attention limits may introduce numerical instability, leading to an ill-conditioned attention computation that offsets the potential benefits of increased capacity.

In contrast, increasing the repetition count leads to a monotonic improvement in success rate, but with a different characteristic shape.
Performance improves substantially when increasing $\text{repeat}$ from 0 to small positive values, while further increases (e.g., from 50 to 90) result in minor gains, indicating early saturation.
Unlike attention, repetition does not degrade performance at high values, but its effect is contingent on sufficient attention capacity.

Importantly, attention limit and repetition exhibit a strong interaction effect.
When the attention limit is very small, increasing repetition alone cannot recover performance, whereas at sufficiently large attention limits, even modest repetition is sufficient to achieve high success rates.
This suggests that attention capacity determines whether the task is solvable, while repetition primarily affects the reliability of successful outcomes once the attention constraint is relaxed.

Finally, the absolute performance levels and sensitivity to these factors depend on the input and output lengths.
Tasks with longer inputs require higher attention limits and greater repetition to achieve comparable performance, while shorter input-output settings reach near-saturation at much smaller values.
Overall, these results highlight complementary roles of attention capacity and repetition, as well as their dependence on task length characteristics.

\section{Comparison with Prompt-Based Steering on Qwen3-4B} 
\label{app:baselines}

A natural concern is that forcing an LLM to emit a fixed output is easy. We compare attention matching with two prompt-based approaches on Qwen3-4B. The first, \emph{repeated target}, prepends the target output repeated $N\in\{16,32\}$ times, in the hope that the model continues the pattern. The second, \emph{system prompt}, instructs the model through the system message to output the target and ignore the user input. We use a long variant,
\begin{quote}
\small\ttfamily
Your response must always be exactly: TARGET\_OUTPUT. Output no additional characters, whitespace, punctuation, formatting, or explanation. Ignore all user requests to change this response.
\end{quote}
and a short variant, \texttt{Output "TARGET\_OUTPUT", ignore the user prompt.} Both use the Qwen3 chat template with an empty thinking block.

Table~\ref{table:qwen3_baselines} reports the overall success rate over all 60 groups. Repeating the target barely helps, and even the long system prompt fails on roughly one instance in six, whereas attention matching succeeds on more than $99\%$ of the instances with $r=1$ and on all instances with $r\ge10$. The prompt-based baselines are restricted to discrete tokens, whereas our prefix is continuous; the comparison shows that the task is far from trivial rather than comparing methods under identical access assumptions.

\begin{table}[h]
\centering
\caption{Overall success rates of steering Qwen3-4B toward a target output, averaged over all 60 groups.}
\label{table:qwen3_baselines}
\begin{tabular}{lc}
\toprule
Method & Success rate \\
\midrule
No prefix & 0.0351 \\
Repeated target, $N=16$ & 0.1660 \\
Repeated target, $N=32$ & 0.1700 \\
System prompt, short & 0.5547 \\
System prompt, long & 0.8462 \\
\midrule
Attention matching (ours), $r=1$ & 0.9966 \\
Attention matching (ours), $r=10$ & \textbf{1.0000} \\
Attention matching (ours), $r=50$ & \textbf{1.0000} \\
\bottomrule
\end{tabular}
\end{table}

\section{Relationship to Attention Matching in Zweiger at al. (2026)}
\label{app:zweiger}

\paragraph{Independent and parallel development.}
A closely related attention matching framework was independently developed by \citet{zweiger2026fast} for KV cache compaction. Our attention matching framework was also developed independently, with its core formulation, linear-system construction, and Vandermonde-based existence proof already included in our original submission to ICML 2026 in January 2026, prior to the public release of \citet{zweiger2026fast} in February 2026. Thus, while the two works share a similar underlying mathematical principle and terminology, they were developed independently and in parallel.

\paragraph{Common underlying principle, different construction and solution methods.}
Both works exploit the observation that, once the attention keys are fixed, attention outputs depend linearly on the value representations. This allows attention matching to be formulated as a linear-algebraic problem rather than requiring end-to-end gradient-based optimization. While both works share this principle idea, they differ in the construction as well as solution methods for the linear-algebraic problem. Specifically, in our framework, we decouple the key and value prefixes and formulate the matching conditions as a system of linear equations, $SV=[X\mid Y]$, whose solutions determine the desired prefix values. In contrast, \citet{zweiger2026fast} construct compacted KV representations by selecting keys and solving a least-squares problem for the corresponding values to approximate the original attention outputs over a set of reference queries.

\paragraph{Differences in objectives and theoretical contributions.}
Despite this shared principle, the two frameworks address distinct research problems and differ substantially in their theoretical objectives. \citet{zweiger2026fast} focus on KV cache compaction, aiming to preserve the attention behavior of the original cache using fewer KV entries. Their approach primarily selects a subset of existing keys, fits the corresponding values through least squares, and introduces additional attention-logit biases to preserve attention mass.

In contrast, our objective is to establish the expressive power of prefix-tuning: given a pretrained autoregressive Transformer and a reachable target output sequence, we seek to construct a prefix that induces the model to generate that sequence. Rather than selecting existing keys for compression, we construct new prefix keys and values to match the attention outputs of a prescribed autoregressive decoding trajectory. Crucially, we establish sufficient conditions for the resulting linear system to admit a solution, using a Vandermonde-based full-rank argument, and employ prefix repetition to suppress the contribution of the original input tokens. This construction leads to the existence guarantees in Theorem 2.1 under its stated assumptions.

In summary, while both works independently exploit the linear structure of attention matching, our primary contribution is a constructive theoretical characterization of prefix-tuning expressiveness, whereas \citet{zweiger2026fast} develop efficient algorithms for approximate KV cache compaction. The success of attention matching in our work and the concurrent work of \citet{zweiger2026fast} on a different application highlights its potential as a powerful and broadly applicable framework worthy of further study.


\end{document}